\documentclass[letterpaper, 10 pt, conference]{ieeeconf}  

\IEEEoverridecommandlockouts                              

\usepackage{microtype}
\usepackage{graphicx}
\usepackage{subcaption}
\usepackage{booktabs} 
\usepackage[utf8]{inputenc}
\usepackage[T1]{fontenc}
\usepackage{parskip}
\usepackage{amsmath,amssymb}
\usepackage{float}
\usepackage[dvipsnames]{xcolor}
\usepackage{comment}
\usepackage{mathrsfs}
\usepackage[ruled,vlined]{algorithm2e}

\newtheorem{theorem}{Theorem}
\newtheorem{proposition}{Proposition}

\newtheorem{corollary}{Corollary}

\newcommand{\N}{\mathcal{N}}
\newcommand{\States}{\mathcal{S}}
\newcommand{\Actions}{\mathcal{A}}
\newcommand{\Probs}{\mathcal{P}}
\newcommand{\law}{\mathcal{L}}
\newcommand{\R}{\mathbb{R}}
\newcommand{\E}{\mathbb{E}}
\newcommand{\Tr}{\mathrm{Tr}}

\newcommand{\ADM}{\mathrm{ADM}}
\newcommand{\aff}{\mathrm{aff}}
\newcommand{\real}{\mathrm{real}}
\newcommand{\simu}{\mathrm{sim}}
\newcommand{\OT}{\mathrm{OT}}
\newcommand{\defeq}{\overset{\Delta}{=}}

\DeclareMathOperator*{\argmin}{arg\,min}

\title{\LARGE \bf
Affine Transport for Sim-to-Real Domain Adaptation}

\author{Anton Mallasto$^{*,1}$ \medskip Karol Arndt$^{*,2}$ \medskip Markus Heinonen$^1$ \medskip Samuel Kaski$^{1,3}$ \medskip Ville Kyrki$^2$
\thanks{\textbf{Preprint}. Work in progress.}
\thanks{$^*$Equal Contribution.}
\thanks{This work was supported by Academy of Finland grants 317020 and 328399. We acknowledge the computational resources provided by the Aalto Science-IT project.
}%
\thanks{$^{1}$Department of Computer         Science, Aalto University, Espoo, Finland}%
\thanks{$^{2}$Intelligent Robotics Group,Department of Electrical Engineering and Automation, Aalto University, Espoo, Finland}%
\thanks{$^{3}$Department of Computer Science, University of Manchester, Manchester, UK
}
\thanks{{\tt\small first.last@aalto.fi}}
}

\begin{document}

\maketitle
\thispagestyle{empty}
\pagestyle{empty}

\begin{abstract}
Sample-efficient domain adaptation is an open problem in robotics.
In this paper, we present affine transport---a variant of optimal transport, which models the mapping between state transition distributions between the source and target domains with an affine transformation.
First, we derive the affine transport framework; then, we extend the basic framework with Procrustes alignment to model arbitrary affine transformations.
We evaluate the method in a number of OpenAI Gym sim-to-sim experiments with simulation environments, as well as on a sim-to-real domain adaptation task of a robot hitting a hockeypuck such that it slides and stops at a target position.
In each experiment, we evaluate the results when transferring between each pair of dynamics domains.
The results show that affine transport can significantly reduce the model adaptation error in comparison to using the original, non-adapted dynamics model.
\end{abstract}

\maketitle

The recent success of reinforcement learning (RL) has inspired its increasing adoption in robotics~\cite{kober2013reinforcement}. Training RL agents on real robotics data, though, is expensive, due to wear and tear, and slow running time. Instead, policies are usually trained in a simulated environment, which is considerably cheaper and safer. Unfortunately, the simulated environments cannot model the real life perfectly, e.g., the sensory data might differ, as well as the dynamics of the system. To combat this gap, sim-to-real transfer attempts to adapt an agent trained in the simulation by utilizing real world data.

Domain adaptation~\cite{wilson2020survey,redko2020survey} is a popular tool for sim-to-real transfer, where a model trained on a \emph{source domain} (cf. simulation) is transferred to a \emph{target domain} (cf. real robot) with less data available~\cite{zhao2020sim}. Recently, optimal transport (OT) has become an increasingly popular tool in domain adaptation~\cite{courty17DA,courty2017joint,redko2017theoretical}.

Optimal transport provides a geometrical framework for the study of probability measures by lifting a \emph{cost function} between samples to a divergence between two distributions. This lift is carried out by learning a \emph{transport plan} (a joint distribution of the two marginals) that minimizes the total cost of transporting the mass of one distribution to another distribution. The transport plan can then be used to e.g., map one distribution to another, and to construct a relation between samples of the two distributions.

Use of OT in domain adaptation relies on the constructed relation: if a distribution (source domain) changes into another (target domain) with minimal cost, this relation tells us how individual samples were changed. For example, a transport plan between discrete distributions can be computed, which is used to map source samples to target through \emph{barycentric projection}, as done by~\cite{courty17DA}. The downfall here, is that to map novel samples, an optimization problem should be solved. On the other hand, one could directly learn a \emph{transport map} between the distributions~\cite{seguy2017large}, which is closer to the approach taken in this paper.

We propose the \emph{affine transport} (AT) framework, which simplifies the OT framework, by considering OT between the normal approximations of two distributions. This way, we can compute the transport map cheaply in closed-form, which can readily be applied to novel samples. AT is closely related to correlation alignment approaches in domain adaptation, such as CORAL~\cite{sun2015return}, where a linear transformation is computed to minimize the Frobenius norm between the covariance matrices. Effectively this transformation maps the two normal approximations to each other, but not with minimal total transportation cost, as AT does. Zhang et al.~\cite{zhang2018aligning} take one step further, and carry out the covariance alignment in a reproducing kernel Hilbert space (RKHS) using the OT approach. Our work takes this approach to the sim-to-real setting. Furthermore, the present work builds on~\cite{zhang2018aligning}, by formalizing the approach as AT, assessing when AT and OT are equivalent, and furthermore providing error bounds for how close our transferred source distribution is to the target distribution. Based on the theoretical results, we also provide a score function, describing how close the two domains are to differing affinely from each other.

The contributions can be summarized as follows:
\begin{itemize}
    \item We formalize optimal transport based correlation alignment as affine transport, and provide theoretical results on when OT and AT are equivalent, and provide upper bounds on their discrepancy.
    \item We apply AT in sim-to-real tasks, showing clear improvement over situations where no transfer is carried out. 
\end{itemize}

\section{Background}\label{sec:background}
\textbf{Optimal Transport.} Let $(X,d)$ be a metric space equipped with a lower semi-continuous \emph{cost function} $c:X\times X \to \mathbb{R}_{\geq 0}$, e.g the Euclidean distance $c(x,y) = \|x-y\|$. Then, the optimal transport (OT) problem between two probability measures $\nu_0, \nu_1 \in \Probs(X)$ is given by the \emph{Kantorovich problem}
\begin{equation}\label{eq:kantorovich}
    \OT_c(\nu_0, \nu_1) = \min_{\gamma\in \ADM(\nu_0,\nu_1)}\E_\gamma[c],
\end{equation}
where $\ADM(\nu_0,\nu_1)$ is the set of joint probabilities with marginals $\nu_0$ and $\nu_1$, and $\E_\nu[f]$ denotes the expected value of $f$ under $\nu$. The optimal $\gamma$ minimizing \eqref{eq:kantorovich} is called the \emph{OT plan}. Denote by $\law(X)$ the law of a random variable $X$. Then, the OT problem extends to random variables $X,Y$ as
\begin{equation}
    \OT_c(X,Y) \defeq \OT_c(\law(X), \law(Y)).
\end{equation}

Assuming that either of the considered measures are \emph{absolutely continuous}, then the Kantorovich problem is equivalent to the \emph{Monge problem}
\begin{equation}
    \OT_c(\mu,\nu) = \min\limits_{T: T_\#\mu = \nu} \E_{X\sim \mu}[c(X,T(X))],
\end{equation}
where the minimizing $T$ is called the \emph{OT map}, and $T_\#\mu$ denotes the \emph{push-forward measure}, which is equivalent to the \emph{law} of $T(X)$, where $X\sim \mu$.  

\textbf{Wasserstein distance.} Let $X$ be a random variable over $\R^d$ satisfying $\E[\|X-x_0\|^2]<\infty$ for some $x_0\in \R^d$, and thus for any $x\in \R^d$. We denote this class of random variables by $\Probs_2(\R^d)$. Then, the $2$-Wasserstein distance $W_2$ between $X,Y\in \Probs_2(\R^d)$ is defined as
\begin{equation}
    W_2(X,Y) = \OT_{d^2}(X, Y)^{\frac{1}{2}},
\end{equation}
where $d(x,y) = \|x-y\|$. The OT maps in this case are characterized by the following theorem
\begin{theorem}[Knott-Smith Optimality Criterion~\cite{smith1987note}]\label{thm:OT_maps}
Let $X\in \Probs_2(\R^d)$.  Furthermore, assume $T(x) = \nabla \phi(x)$ for a convex function and that $T(X)\in\Probs_2(\R^d)$. Then $T$ is the unique OT map between $\mu$ and $T_\#\mu$.
\end{theorem}

When working with the $2$-Wasserstein case, we can without loss of generality assume centered distributions, as
\begin{equation}\label{eq:wasserstein_split}
\begin{aligned}
    W_2^2(X,Y) =& W_2^2(X-\E[X], Y-\E[Y])\\
    &+\|\E[X]-\E[Y]\|^2,
\end{aligned}
\end{equation}
splitting the $2$-Wasserstein distance into two independent terms concerning the $L^2$ distance between the means and the $2$-Wasserstein distance between the centered measures. Furthermore, if we have an OT map $T'$ between $X-\E[X]$ and $Y-\E[Y]$, then
\begin{equation}\label{eq:linear_to_affine}
    T(x) = T'(x-\E[X]) + \E[Y],
\end{equation}
is the OT map between $X$ and $Y$.

One of the rare cases where the $2$-Wasserstein distance admits a closed form solution, is between two multivariate Gaussian distributions $\nu_i=\N(\mu_i,\Sigma_i)$, $i=1,2$ which is given by \cite{givens84,dowson82,olkin82,knott84}
\begin{equation}\label{eq:gaus_was_d}
\begin{aligned}
    W_2^2(\nu_1, \nu_2) =& \|\mu_1-\mu_2\|^2 + \Tr(\Sigma_1) + \Tr(\Sigma_2)\\
    &- 2 \Tr\left(\Sigma_2^\frac{1}{2} \Sigma_1 \Sigma_2^\frac{1}{2}\right)^\frac{1}{2}.
\end{aligned}
\end{equation}
Furthermore, the OT map between $\nu_1$ and $\nu_2$ is given by the affine map
\begin{equation}
\begin{aligned}
    F(x) = Ax + b,\quad A &= \Sigma_2^\frac{1}{2}\left(\Sigma_2^\frac{1}{2}\Sigma_1\Sigma_2^\frac{1}{2}\right)^{-\frac{1}{2}}\Sigma_2^\frac{1}{2},\\
    b &= \mu_2 - A\mu_1,
\end{aligned}
\end{equation}
where $\Sigma^\frac{1}{2}$ denotes the matrix square-root. When restricted to the set of Gaussians, the $2$-Wasserstein distance results from a Riemannian metric~\cite{takatsu11}. The Gaussians are special to the $2$-Wasserstein distance also due to the following theorem.

\begin{theorem}[Gelbrich bound~\cite{gelbrich1990formula}]\label{thm:gaussians_lower_bound_2was}
Let $X,Y\in \Probs_2(\R^d)$ and $N_X,N_Y$ be their normal approximations sharing the same means and covariances. Then,
\begin{equation}
    W_2(N_X,N_Y) \leq W_2(X,Y).
\end{equation}
\end{theorem}

\textbf{Sim-to-Real Transfer.} 
Assume a source domain with a dynamics model $g_s:\States \times \Actions \to \States$, mapping a state-action pair to the next state, and a corresponding target domain with a dynamics model $g_t$. In a typical setting, the source domain would be given by simulation, and the target domain by real world data. Then, assume that there exists an underlying \emph{transfer map} $T^*$, so that $T^*(s, a, g_s(s,a)) = g_t(s,a)$. Now, given state and action pairs $(s_i,a_i)$, $i=1,...,N$, data is collected from both domains as collections of triplets $X_s = \{(s_i,a_i,g_s(s_i,a_i)\}_{i=1}^N$, and similarly for $X_t$. The goal is then to infer $T^*$ using these two datasets.

\textbf{Limitations of Optimal Transport.} As we apply OT to infer the map $T$ (in practice between the empirical measures defined by $X_t$ and $X_s$), some restrictions arise due to the nature of OT, formalized by the Brenier factorization theorem:

\begin{theorem}[Brenier Factorization~\cite{brenier1991polar}]\label{thm:brenier_fac}
Let $\Omega\subset \R^n$ be bounded smooth domain, $T:\Omega \to \R^n$ a Borel map which does not map positive volume into zero volume. Then, $T$ uniquely decomposes as
\begin{equation}
    T=t \circ u,
\end{equation}
where $u:\Omega \to \Omega$ is volume preserving and $t=\nabla \psi$ is the gradient of a convex function $\psi:\R^n \to \R$.
\end{theorem}

In the linear world, this translates into the polar decomposition of a matrix: given an invertible matrix $A$, we can write it uniquely as $A=PU$, where $P$ is symmetric positive definite and $U$ orthogonal.

According to Theorem~\ref{thm:OT_maps}, for the $2$-Wasserstein distance, OT maps are precisely the gradient of convex maps. Therefore, we can only learn $f$ modulo volume preserving functions with OT.

\section{Affine Transport}
In the following, we define the \emph{affine transport} (AT) framework, a simplification of OT, which we will apply to the Sim-to-Real problem. It is motivated by cheap, closed-form computations that are simple to implement, and allows for theoretical guarantees.

\textbf{Affine Transport.} Denote by $N_X$ the normal approximation of a random variable $X$. Then, the \emph{affine transport map} (AT map) $T_\aff$ between a \emph{source} $X$ and a \emph{target} $Y$ is given by the OT map between $N_X$ and $N_Y$, that is, the affine transformation
\begin{equation}
\begin{aligned}
     T_\aff(x) &= Ax + b\\
     A &= \Sigma(Y)^\frac{1}{2}\left(\Sigma(Y)^\frac{1}{2} \Sigma(X) \Sigma(Y)^\frac{1}{2}\right)^{-\frac{1}{2}}\Sigma(Y)^\frac{1}{2}\\
     b &= \mu(Y)-A\mu(X), 
\end{aligned}
\end{equation}
where $\Sigma(X)$ is the covariance matrix and $\mu(X)$ the mean of $X$. If needed, we will denote the AT map by $T_\aff[X,Y]$ to emphasize the source and target.

\textbf{Connections with Optimal Transport.}  Naturally if $X$ and $Y$ are already Gaussian, the AT and OT maps are exactly the same. But are there other cases where AT and OT coincide? Yes, whenever $Y$ is a positive semi-definite affine transform of $X$. In the following, we consider centered distributions and linear maps, but due to \eqref{eq:wasserstein_split} and \eqref{eq:linear_to_affine}, the results can be extended to the non-centered case by replacing linear maps with affine maps. We start with a corollary of the Knott-Smith optimality criterion.

\begin{corollary}\label{cor:affine_ot_map}
Let $X,Y\in \Probs_2(\R^d)$ be centered, and assume that $Y=TX$, where $T$ is a positive semi-definite matrix. Then, $T$ is the optimal transport map from $X$ to $Y$.
\end{corollary}
\begin{proof}
Apply Theorem~\ref{thm:OT_maps} to the convex $\phi(x)= x^T T x$, where $T$ is positive semi-definite.
\end{proof}

The next result then states when AT and OT coincide. Furthermore, it allows computing the $2$-Wasserstein distance between arbitrary $X$ and $Y$ in closed-form using their normal approximations, given that $X$ and $Y$ differ by an affine transformation.
\begin{theorem}\label{thm:affine_OT}
Let $X,Y\in \Probs_2(\R^d)$ be centered and $Y = TX$ for a positive definite matrix $T$. Then $T$ is also the OT map between $N_X$ and $N_Y$, and $W_2(N_X,N_Y) = W_2(X,Y)$. That is, $T = T_\aff[X,Y]$.
\end{theorem}
\begin{proof}
Corollary~\ref{cor:affine_ot_map} states that $T$ is an OT map, and 
\begin{equation}
    \Sigma(T N_X) = T\Sigma(X)T = \Sigma(Y).
\end{equation}
Therefore, $T N_X = N_Y$, and by Theorem~\ref{thm:OT_maps}, $T$ is the OT map between $N_X$ and $N_Y$. Finally, we compute
\begin{equation}
    \begin{aligned}
    W^2_2(N_X,N_Y) =& \Tr[\Sigma(X)] + \Tr[T\Sigma(X)T]\\
    &- 2\Tr[T^\frac{1}{2}\Sigma(X) T^\frac{1}{2}]\\
    =& \argmin\limits_{T:T(X) = Y}\E_X[\|X-T(X)\|^2]\\
    =& W_2^2(X,Y).
    \end{aligned}
\end{equation}
\end{proof}

\textbf{Error Bounds for Affine Transport.} If $X$ and $Y$ differ by a more complicated transformation than an affine one, how well does $T_\aff X$ match $Y$? We answer this below, by first bounding the difference between the $2$-Wasserstein distance of $X,Y$ and the $2$-Wasserstein distance between their normal approximations.

\begin{proposition}\label{prop:affine_vs_standard_OT}
Let $X,Y\in \Probs_2(\R^d)$ and $N_X, N_Y$ be their normal approximations. Then, we have the bound
\begin{equation}
    \left| W_2(N_X, N_Y) - W_2(X,Y)\right| \leq \frac{2\Tr\left[\left(\Sigma(X)\Sigma(Y)\right)^\frac{1}{2}\right]}{\sqrt{\Tr[\Sigma(X)] + \Tr[\Sigma(Y)]}}.
\end{equation}
\end{proposition}
\begin{proof}
By Theorem~\ref{thm:gaussians_lower_bound_2was}, we have $W_2(N_X,N_Y) \leq W_2(X,Y)$. On the other hand,
\begin{equation}
\begin{aligned}
    W^2_2(X,Y) &= \min\limits_{\gamma \in \ADM(X,Y)} \int_{\R^d \times \R^d} \|x-y\|^2 d\gamma(x,y)\\
    &\leq \int_{\R^d \times \R^d}\left(\|x\|^2 + \|y\|^2\right) d\gamma(x,y)\\
    &= \Tr[\Sigma(X)] + \Tr[\Sigma(Y)].
\end{aligned}
\end{equation}
Combining the above inequalities, we get
\begin{equation}\label{eq:affine_vs_standard_OT_eq1}
\begin{aligned}
    &\left|W_2(N_X,N_Y) - W_2(X,Y)\right|\\
    \leq& \left| \sqrt{\Tr[\Sigma(X)] + \Tr[\Sigma(Y)]} - W_2(N_X,N_Y)\right|.
\end{aligned}
\end{equation}
Let $a = \Tr[\Sigma(X)] + \Tr[\Sigma(Y)]$, and so $W^2_2(N_X,N_Y) = a - b$, where $b = 2\Tr\left[\left(\Sigma(X)\Sigma(Y)\right)^\frac{1}{2}\right]$. Then the RHS of \eqref{eq:affine_vs_standard_OT_eq1} can be written as
\begin{equation}
    \left|\sqrt{a} - \sqrt{a-b}\right| = \frac{|a - (a-b)|}{\sqrt{a} + \sqrt{a-b}} \leq \frac{b}{\sqrt{a}},
\end{equation}
where the inequality follows from positivity of $W_2(N_X,N_Y)=\sqrt{a-b}$. This implies the claim.
\end{proof}

Proposition~\ref{prop:affine_vs_standard_OT} implies the following $2$-Wasserstein bound for the normal approximation of a random variable.
\begin{corollary}\label{cor:normal_approximation_error}
Let $X\in \Probs_2(\R^d)$ and $N_X$ be its normal approximation. Then
\begin{equation}
    W_2(N_X,X) \leq \sqrt{2}\Tr[\Sigma(X)]^\frac{1}{2}.
\end{equation}
\end{corollary}
\begin{proof}
Apply Proposition~\ref{prop:affine_vs_standard_OT} with $X=X$ and $Y=N_X$.
\end{proof}

Finally, we show the following error bound for AT. which only depends on the target.

\begin{proposition}\label{prop:affine_OT_error}
    Let $X,Y\in \Probs_2(\R^d)$ and $T_\aff$ be the AT map from $X$ to $Y$. Then, 
    \begin{equation}
        W_2(T_\aff X,Y) \leq \sqrt{2}\Tr\left[\Sigma(Y)\right]^\frac{1}{2}.
    \end{equation}
\end{proposition}
\begin{proof}
By a direct computation
\begin{equation}
\begin{aligned}
    W_2^2(T_\aff X, Y) &\leq \Tr[\Sigma(T_\aff X)] + \Tr[\Sigma(Y)]\\
    &= 2\Tr[\Sigma(Y)].
\end{aligned}
\end{equation}
\end{proof}

The bound given in Proposition~\ref{prop:affine_OT_error} lets us define the following \emph{affinity score}
\begin{equation}
    \rho_\aff(X,Y) = 1 - \frac{W_2(T_\aff X,Y)}{\sqrt{2}\Tr[\Sigma(Y)]^\frac{1}{2}},
\end{equation}
describing how much $Y$ differs from being a positive-definite affine transformation of $X$, as $0\leq \rho_\aff(X,Y) \leq 1$, and $\rho_\aff(X,Y)=1$ when $X$ and $Y$ differ by an affine transformation, and $\rho_\aff(X,Y)=0$ when the OT plan between $T_\aff X$ and $Y$ is the independent distribution.

\begin{figure}
    \centering
    \includegraphics[width=\linewidth]{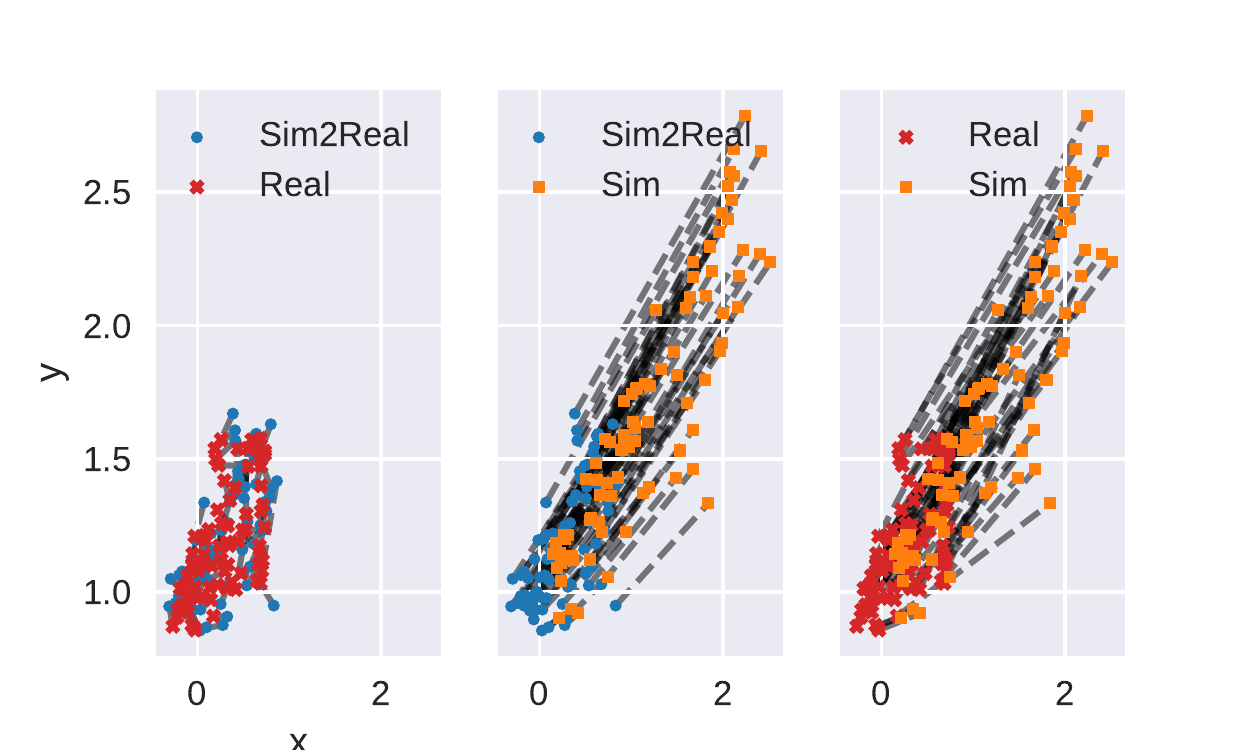}
    \caption{ Sim2Real transfer illustrated for the blue puck simulated with \textit{isotr\_low\_offc} settings.}
    \label{fig:transfer_example}
\end{figure}

\textbf{Affine Transport for Domain Adaptation.} Now, given source and target data $X_s$ and $X_t$, respectively, we can simply apply affine transport and use $T_\aff[X_s, X_t]$ as the transfer map. However, as discussed in Sec.~\ref{sec:background}, this way we can only learn the underlying transfer map up to a volume preserving measure. 

More concretely, if $X_t=T^*(X_s)$, where $T^*(x) = A^*x + b^*$ is affine, we can write the polar decomposition $A^*=PU$. If the orthogonal $U$ is the identity, then AT recovers the underlying transfer map $A^*$. Otherwise, we need a way to take into account the orthogonal part. This motivates us to 'preprocess' the data by applying \emph{Procrustes alignment}~\cite{kendall1989survey}: given matrices $A,B$, find the orthogonal matrix $R$ such that
\begin{equation}
    R = \argmin\limits_{R:R^TR=I} \|RA-B\|.
\end{equation}
In practice $R$ is easy to compute. Let $M=BA^T$ with the singular-value decomposition $M=V_1 D V_2^T$, then $R = V_1V_2^T$. The transfer learning method resulting from applying Procrustes alignment and AT is summarized in Algorithm~\ref{alg:method}. Note that although we do recover a pair of a positive definite and an orthogonal matrix, there are no quarantees that  $(T_\aff, R) = (P,U)$, i.e, we do not necessarily recover the polar decomposition for the underlying transfer map.

\begin{algorithm}\label{alg:method}
\SetKwInOut{Input}{Input}
\SetKwInOut{Output}{Output}
\SetAlgoLined
\Input{Source $X_s$, target $X_t$. Matrices with columns of the form $(s_i, a_i, g(s_i,a_i))$.}
\Output{Transfer map $T$.}
$V_1, D , V_2^T \gets \mathrm{SVD}(X_t,X_s^T)$\;
$R \gets V_1V_2^T$\;
$T' \gets T_\aff[RX_s, X_t]$\;
 $T \gets T' \circ R$\;
 \caption{Sim-to-Learn Transfer with AT}
 \label{alg:method}
\end{algorithm}

\section{Experiments}

\subsection{Gym environments}
In order to assess the method's suitability for domain adaptation, we used the \textit{Hopper} and \textit{HalfCheetah} environment from OpenAI Gym~\cite{brockman16gym}.
In these environments, we randomize the mass and length of each link, randomly invert the direction of each joint, and disable random joints.
The state space in Hopper is 11-dimensional and contains the positions and velocities of each degree of freedom (the rotation angles of each of the 3 joints, the rotation of the whole system, and 2 translations), and the action space is 3-dimensional and consists of torques applied to each joint.
The Ant state space is similar in structure, with 27-dimensional state space and 8-dimensional action space.
Both state spaces ignore the absolute position of the agent in the world (the $x$ position in case of Hopper and $x$ and $y$ in Ant), but include the velocity along these axes.

Domain adaptation is done through applying randomizations to various environment parameters; more specifically, we randomize link masses, joint length (as shown in Figure~\ref{fig:hopz}), and by disabling or inverting the direction of certain joints.
The randomization ranges are presented in Table~\ref{tab:mujoco_rand_ranges}.

The results of this evaluation (using the Hopper and Ant environments are presented in Figures~\ref{fig:hopper_res} and~\ref{fig:ant_res}.
Based on these results, we can observe that using affine transport, despite its simplicity, significantly improves the state prediction performance, with the largest relative reduction apparent for the \textit{inverted} domain.

\begin{figure}
    \centering
    \begin{subfigure}{0.3\linewidth}
    \centering
    \includegraphics[width=\textwidth]{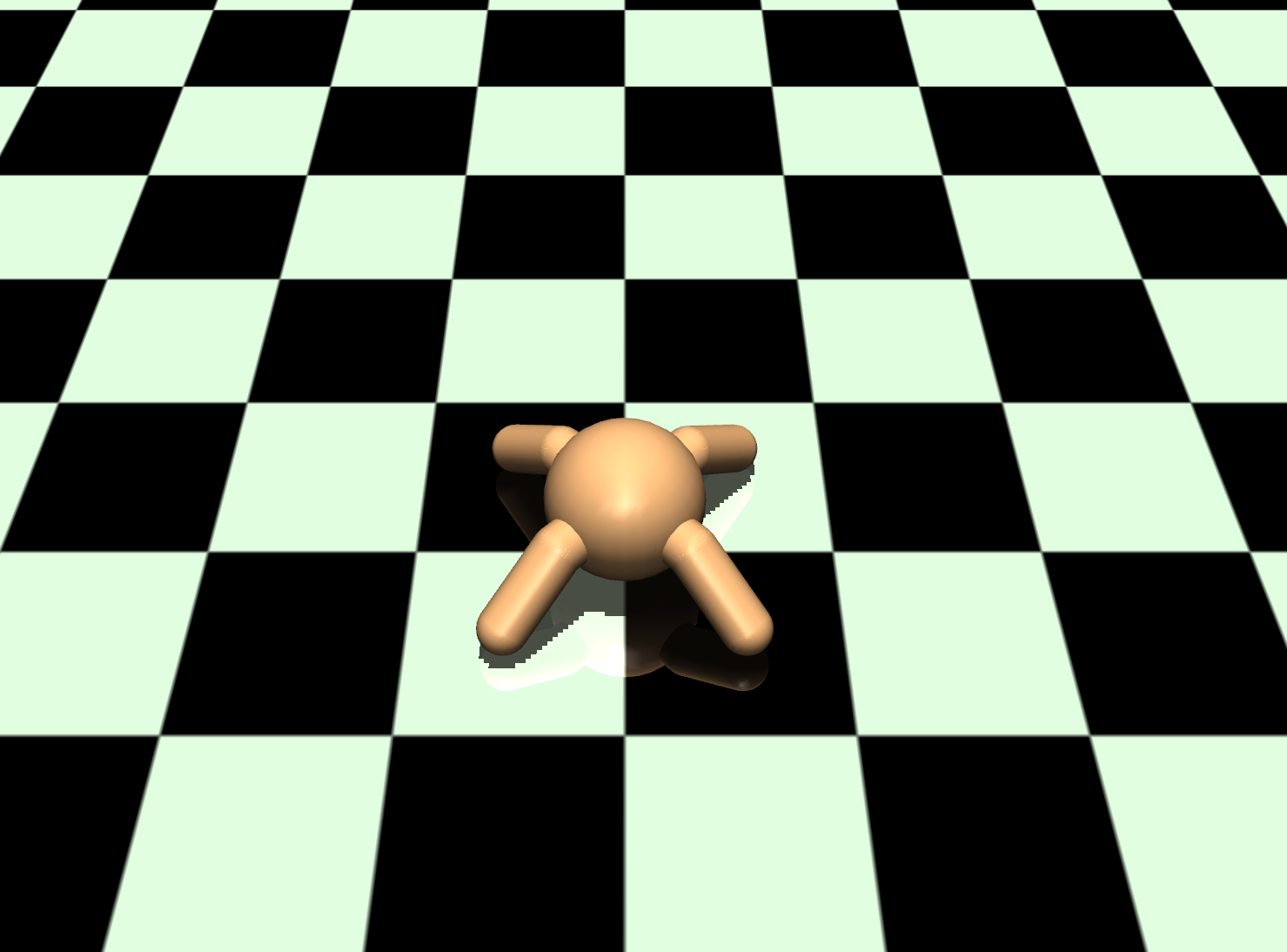}
    \caption{}
    \label{fig:smallant}
    \end{subfigure}
    \begin{subfigure}{0.3\linewidth}
    \centering
    \includegraphics[width=\textwidth]{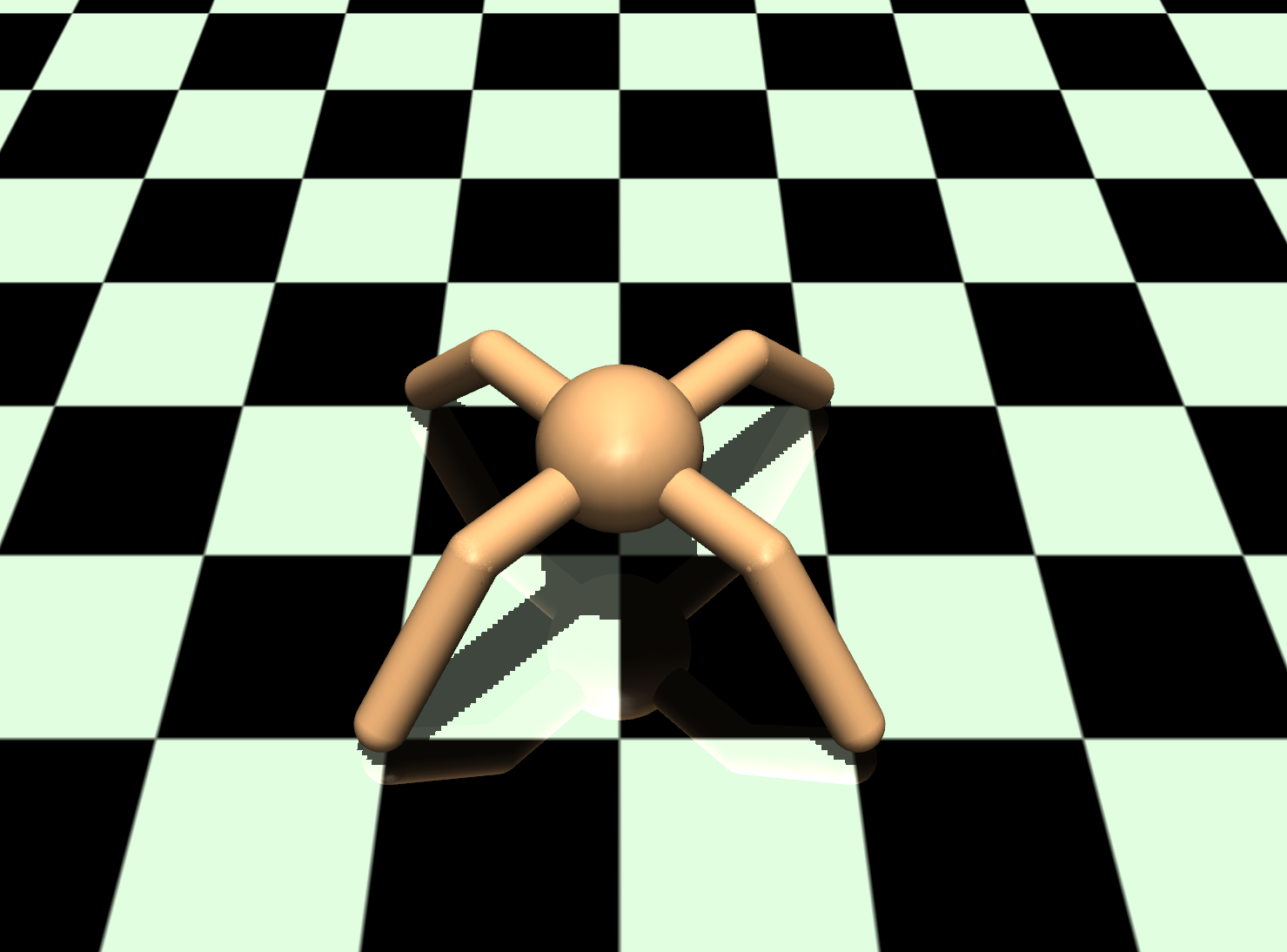}
    \caption{}
    \label{fig:medant}
    \end{subfigure}
    \begin{subfigure}{0.3\linewidth}
    \centering
    \includegraphics[width=\textwidth]{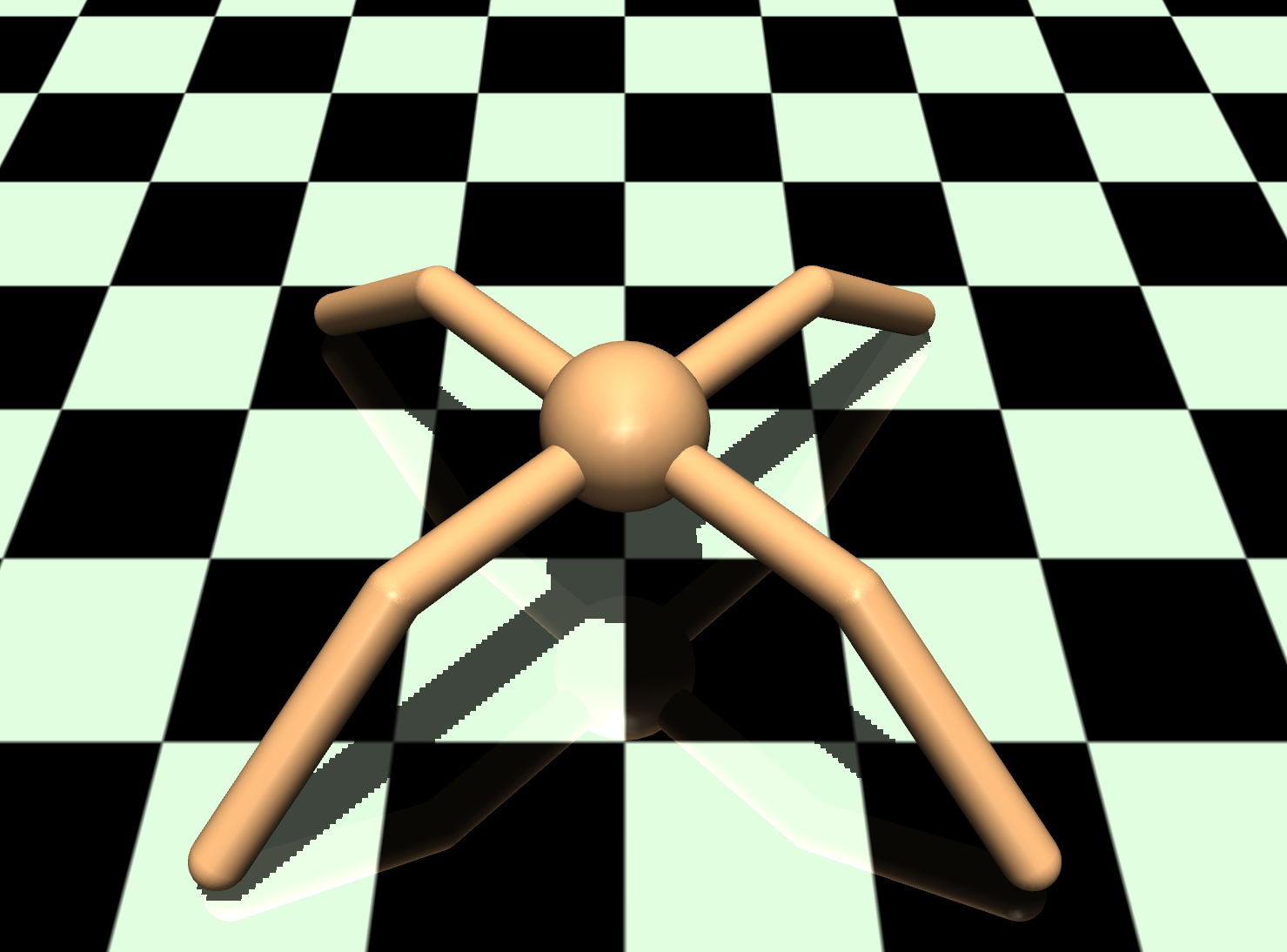}
    \caption{}
    \label{fig:bigant}
    \end{subfigure}
    \begin{subfigure}{0.35\linewidth}
    \centering
    \includegraphics[width=\textwidth]{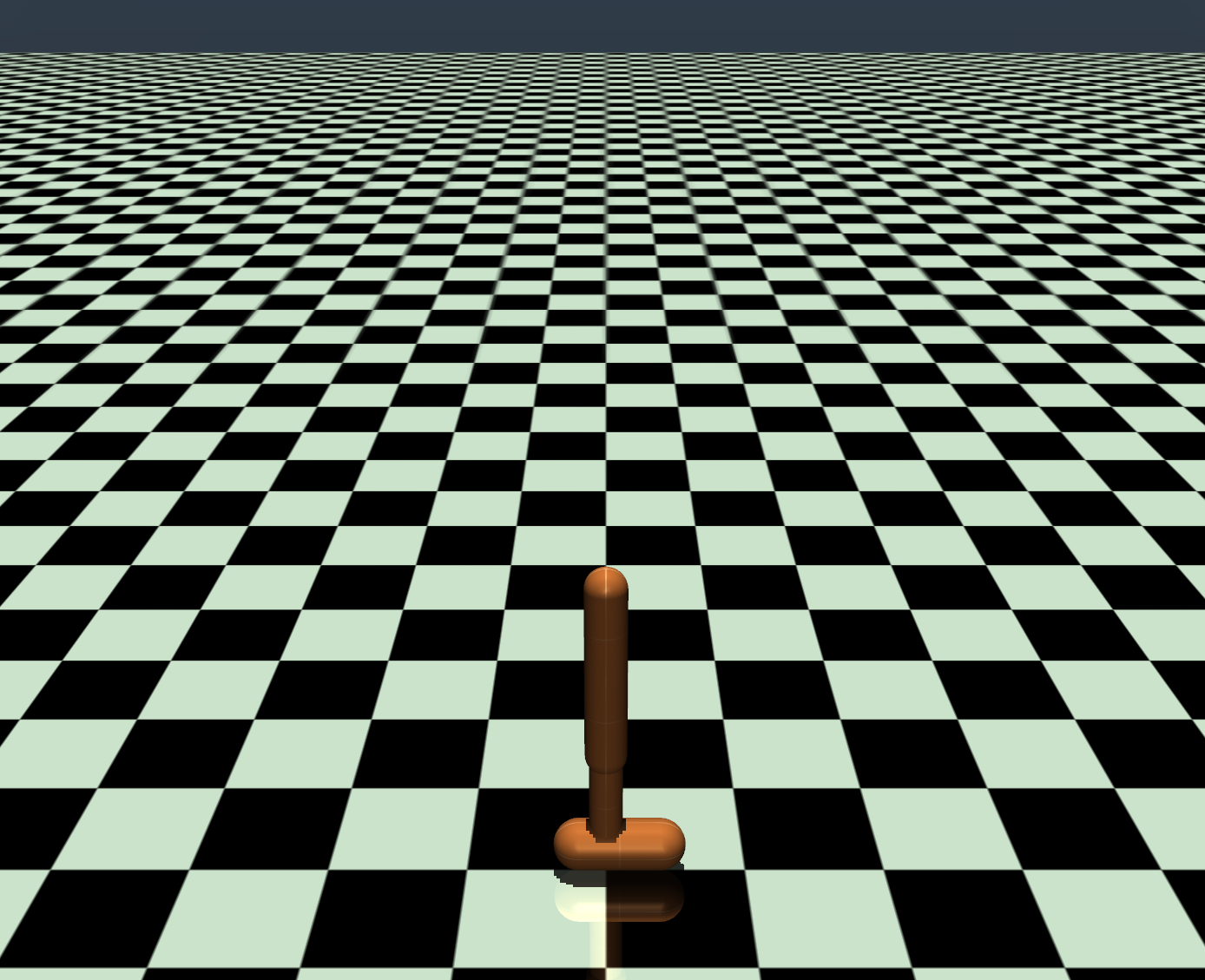}
    \caption{}
    \label{fig:smallhop}
    \end{subfigure}
    \begin{subfigure}{0.35\linewidth}
    \centering
    \includegraphics[width=\textwidth]{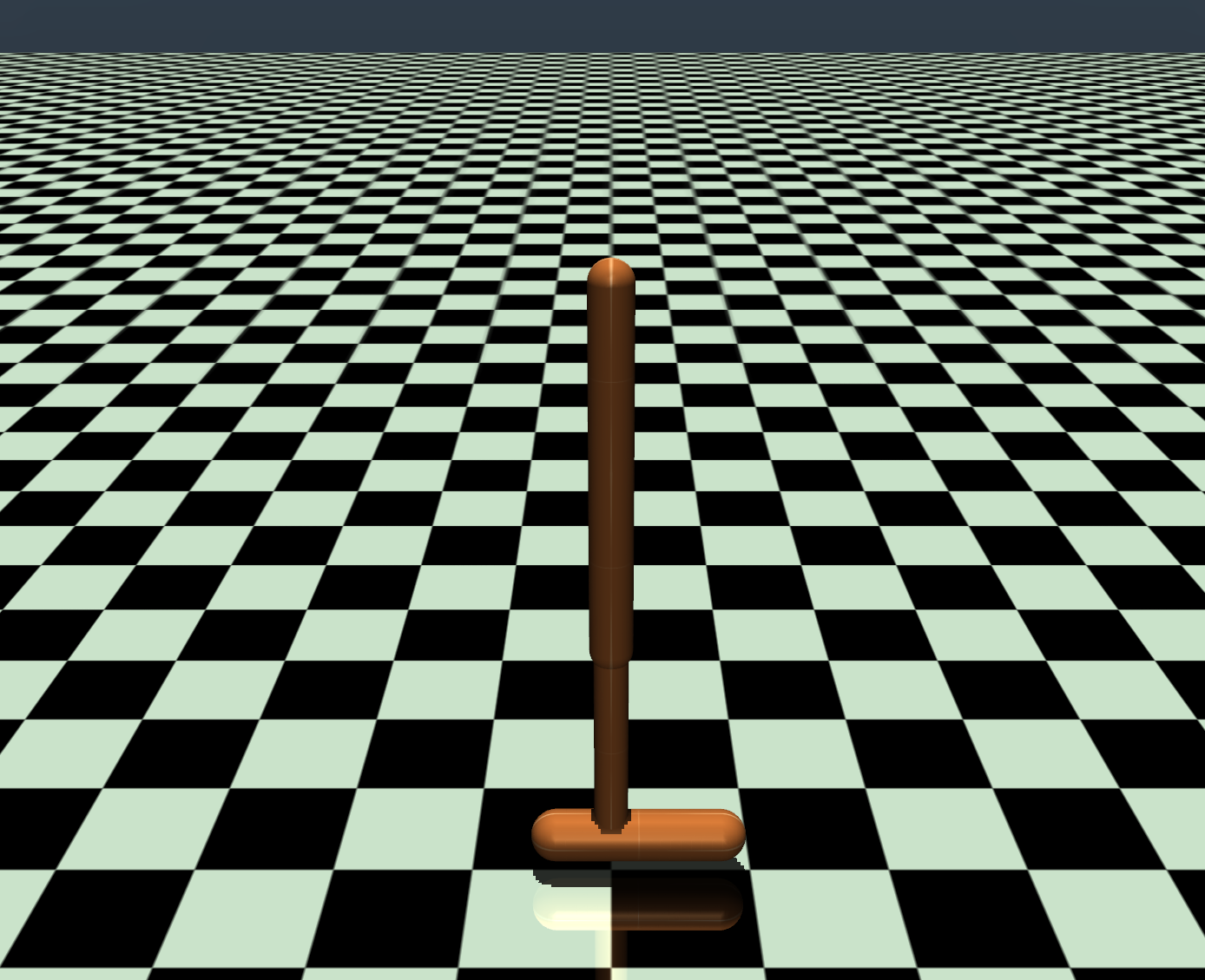}
    \caption{}
    \label{fig:medhop}
    \end{subfigure}
    \begin{subfigure}{0.25\linewidth}
    \centering
    \includegraphics[width=\textwidth]{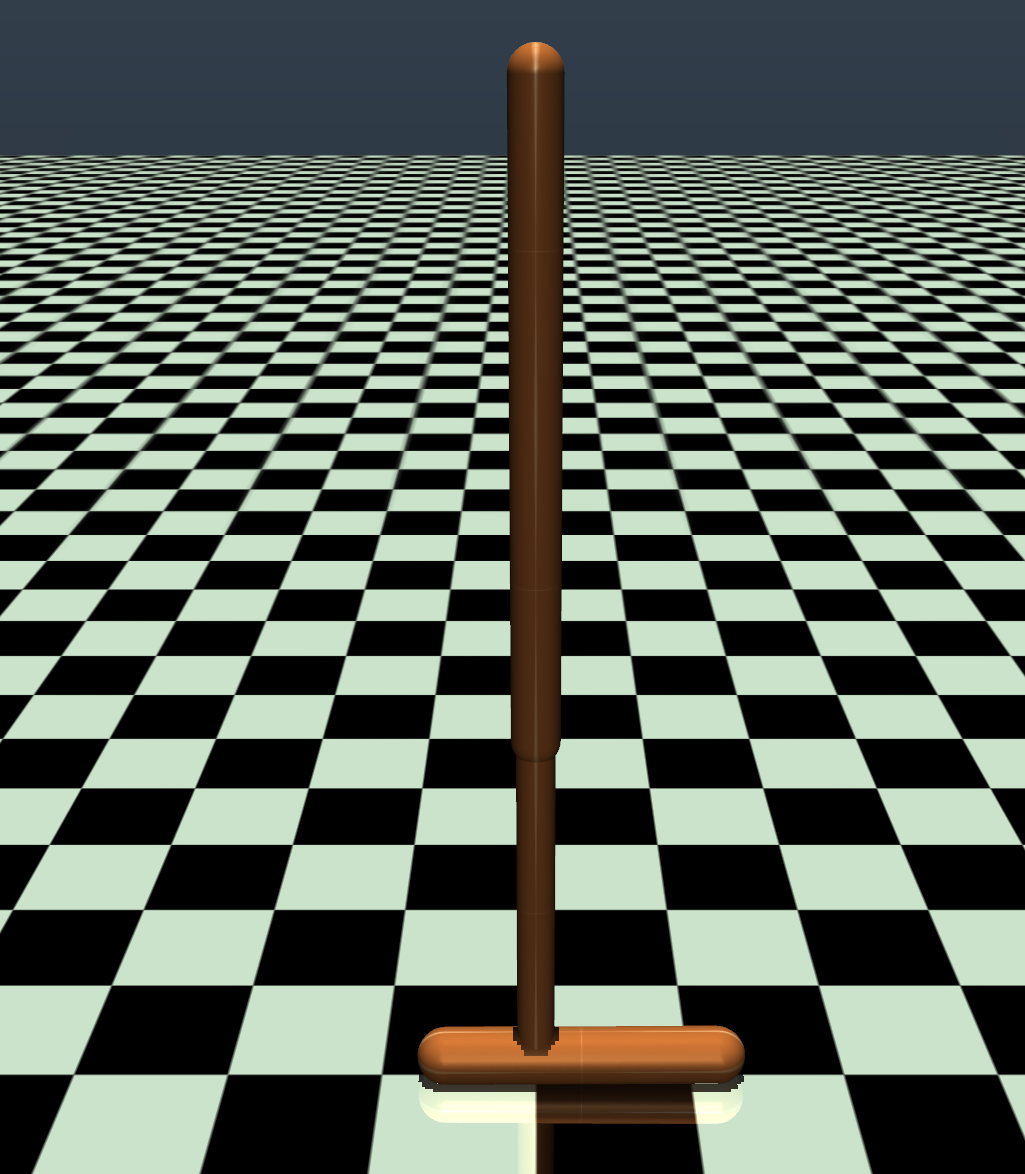}
    \caption{}
    \label{fig:bighop}
    \end{subfigure}
    \caption{The \textit{Ant} (\subref{fig:smallant}, \subref{fig:medant}, \subref{fig:bigant}) and \textit{Hopper} (\subref{fig:smallhop}, \subref{fig:medhop}, \subref{fig:bighop}) environments with different joint lengths used for evaluation}
    \label{fig:hopz}
\end{figure}
\begin{table}[]
    \centering
    \begin{tabular}{c|c|c|c|c}
         Conditions &  Relative size & Relative mass & Inverted joints & Disabled joints \\
         \hline
         Normal & 1 & 1 & None & None \\
         Light & 1 & 0.5 & None & None \\
         Heavy & 1 & 2.0 & None & None \\
         Long & 1.5 & 1 & None & None \\
         Short & 0.5 & 1 & None & None \\
         Some inv. & 1 & 1 & 2, 4, 6, ... & None \\
         Inverted & 1 & 1 & All & None \\
         Some off & 1 & 1 & None & 2, 4, 6, ... \\
    \end{tabular}
    \caption{Randomization ranges for the MuJoCo environments}
    \label{tab:mujoco_rand_ranges}
\end{table}


\begin{figure}
    \centering
    \begin{subfigure}{1.0\linewidth}
        \centering
        \includegraphics[width=.8\linewidth]{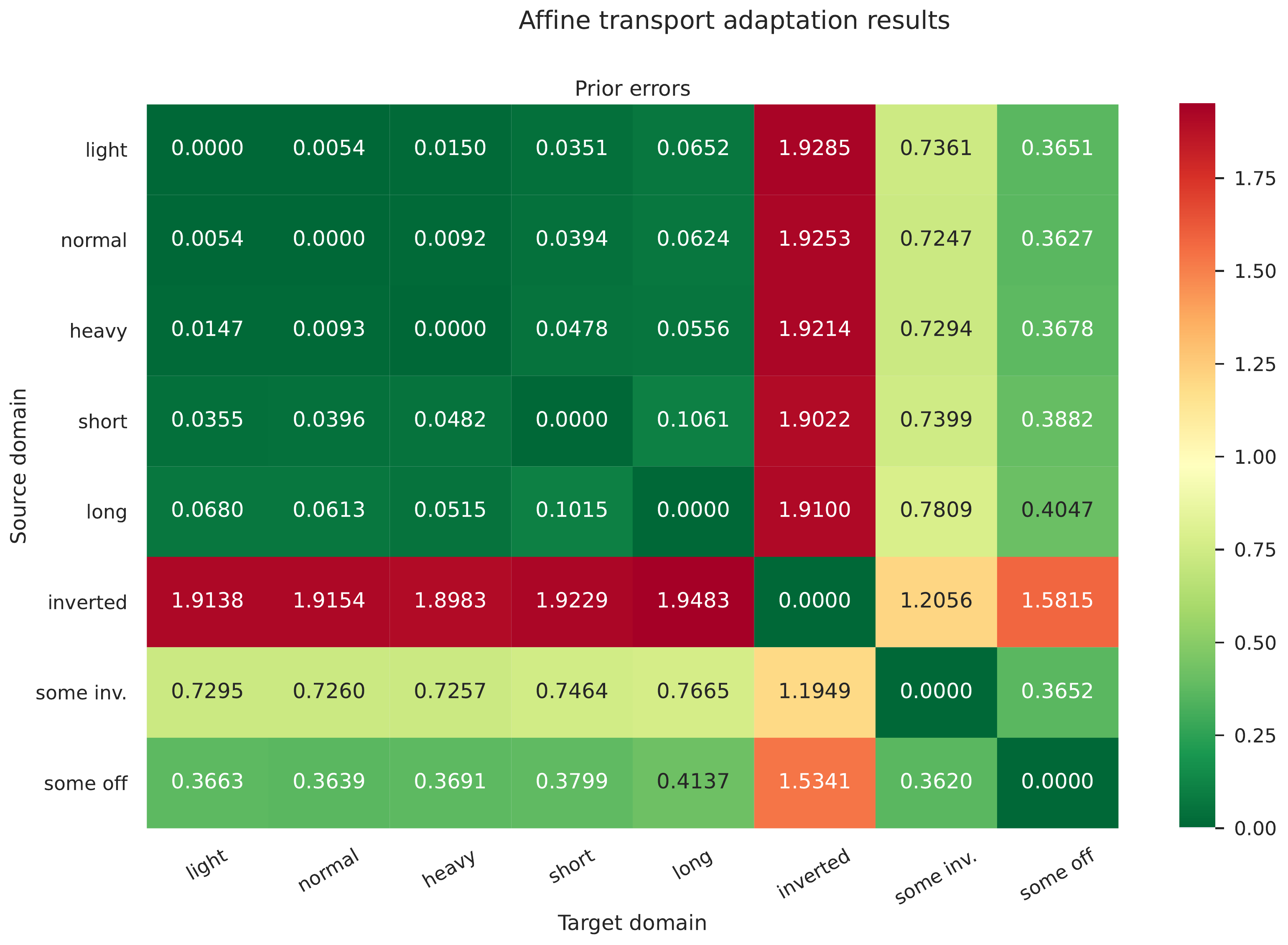}
        \caption{}
        \label{fig:ant_prior}
    \end{subfigure}
    \begin{subfigure}{1.0\linewidth}
        \centering
        \includegraphics[width=.8\linewidth]{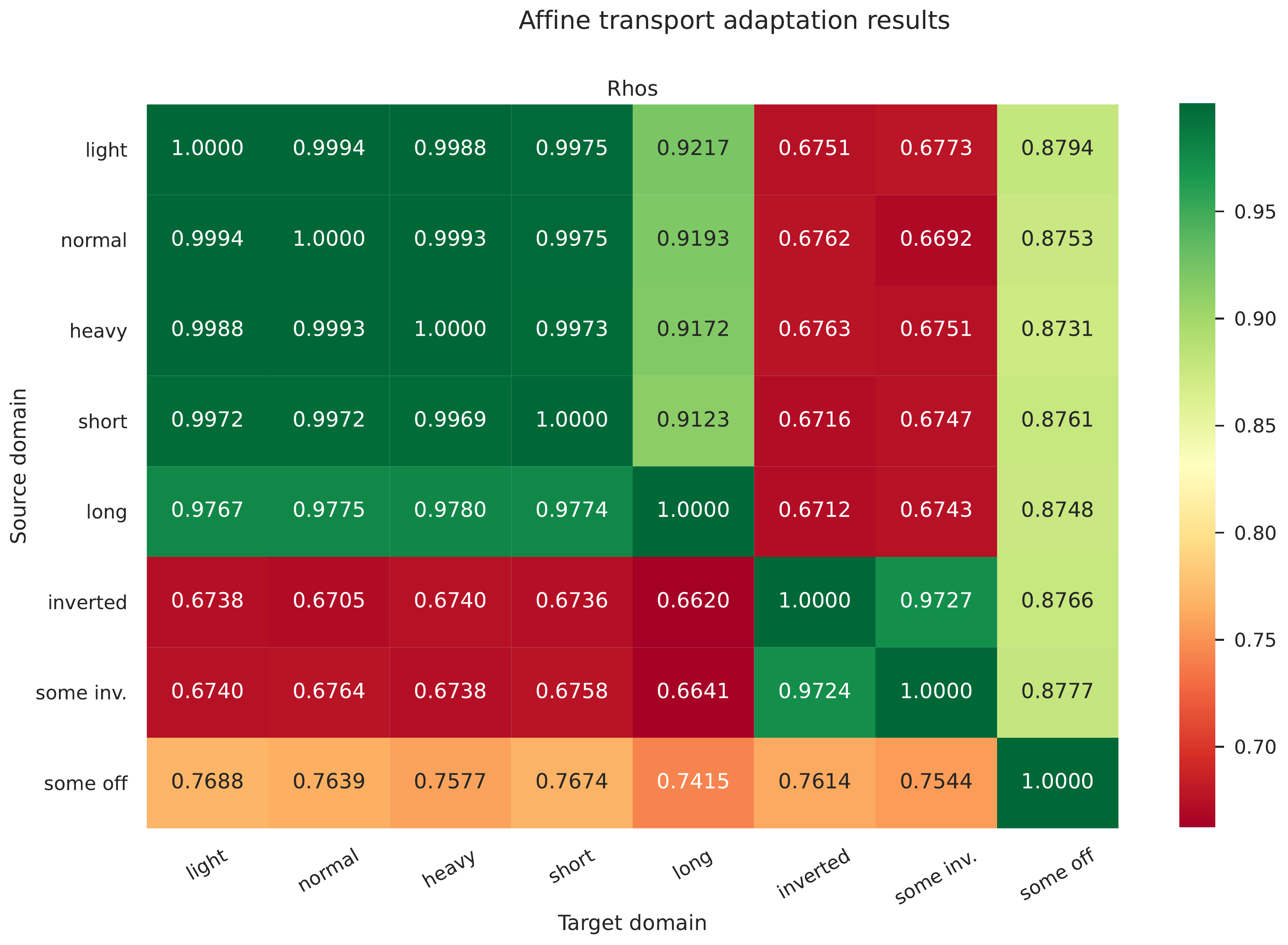}
        \caption{}
        \label{fig:ant_rhos}
    \end{subfigure}
    \begin{subfigure}{1.0\linewidth}
        \centering
        \includegraphics[width=.8\linewidth]{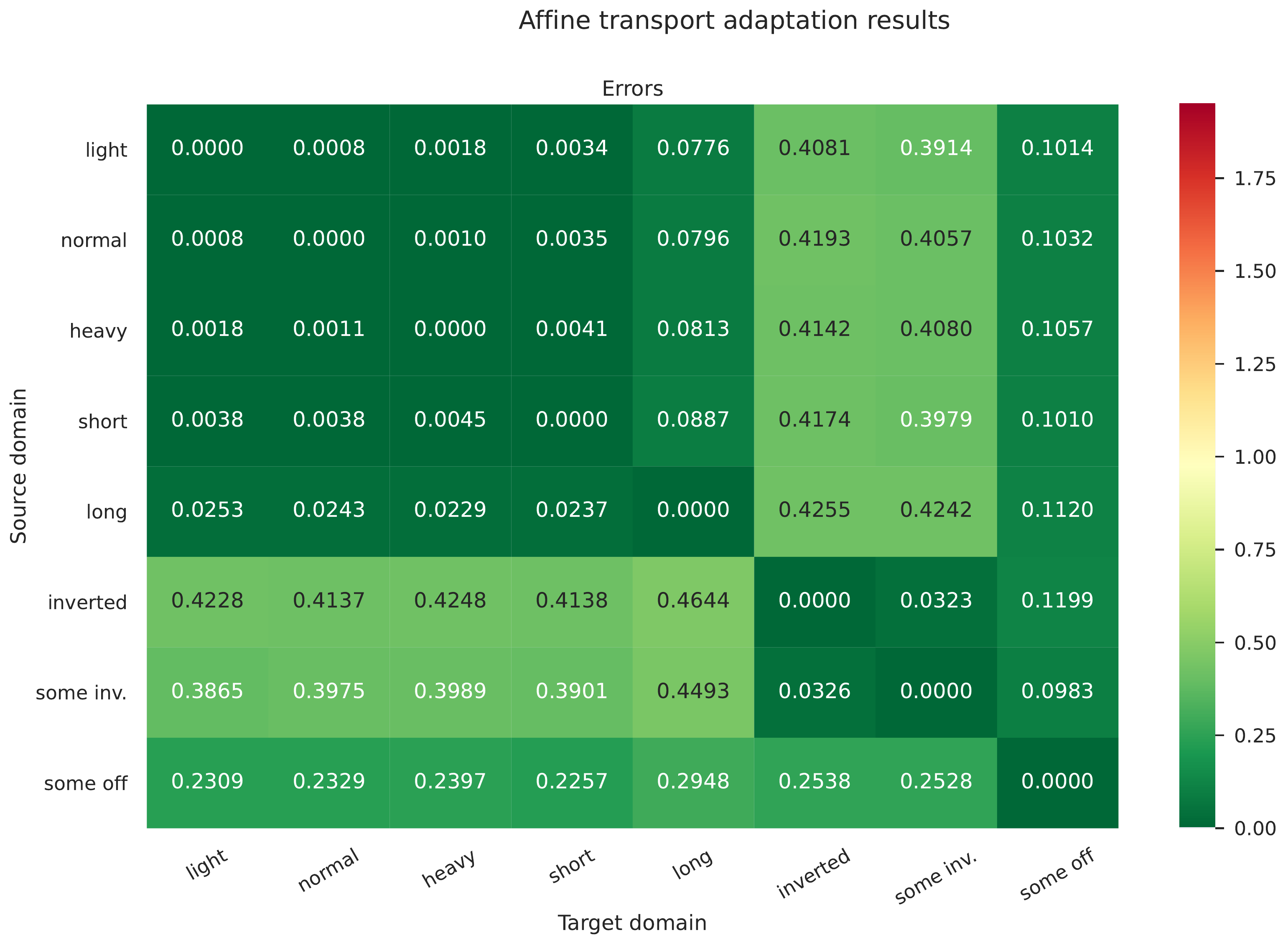}
        \caption{}
        \label{fig:ant_errors}
    \end{subfigure}
    \caption{Results in the Ant environment: before adaptation~(\subref{fig:ant_prior}), the $\rho_{\aff}$ values ~(\subref{fig:ant_rhos}), and the errors after adaptation~(\subref{fig:ant_errors})}
    \label{fig:ant_res}
\end{figure}

\begin{figure}
    \centering
    \begin{subfigure}{1.0\linewidth}
        \centering
        \includegraphics[width=.8\linewidth]{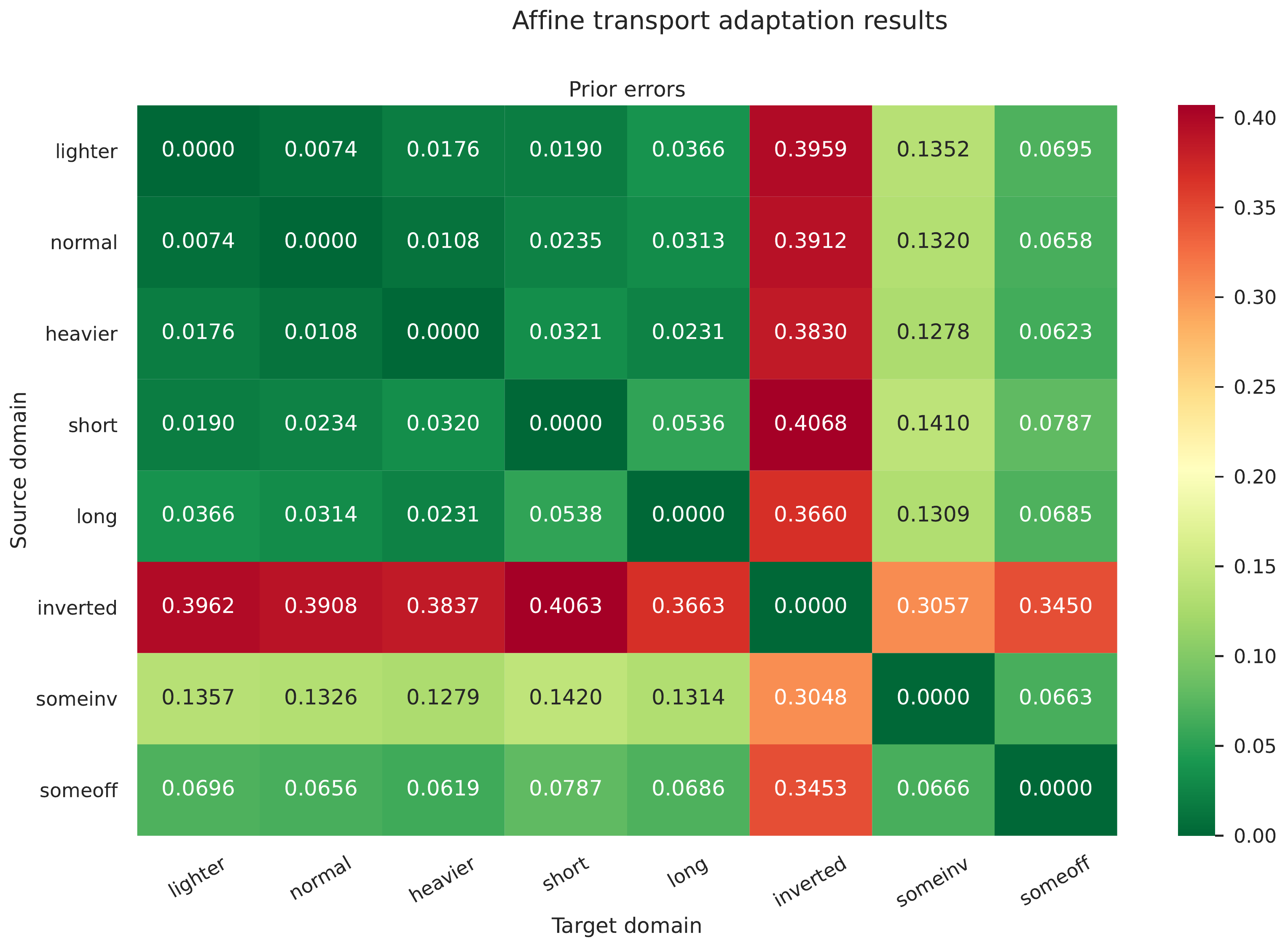}
        \caption{}
        \label{fig:hopper_prior}
    \end{subfigure}
    \begin{subfigure}{1.0\linewidth}
        \centering
        \includegraphics[width=.8\linewidth]{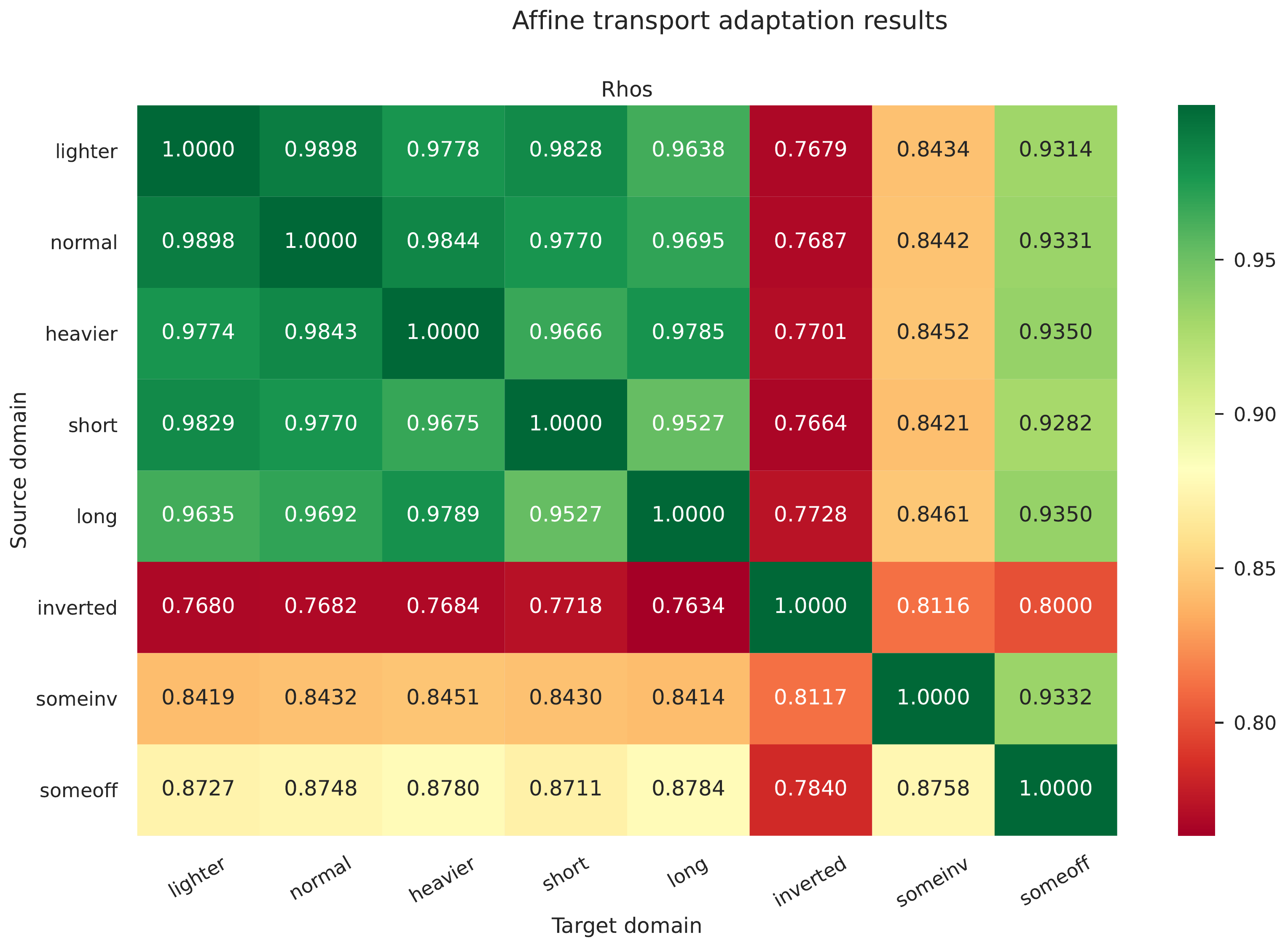}
        \caption{}
        \label{fig:hopper_rhos}
    \end{subfigure}
    \begin{subfigure}{1.0\linewidth}
        \centering
        \includegraphics[width=.8\linewidth]{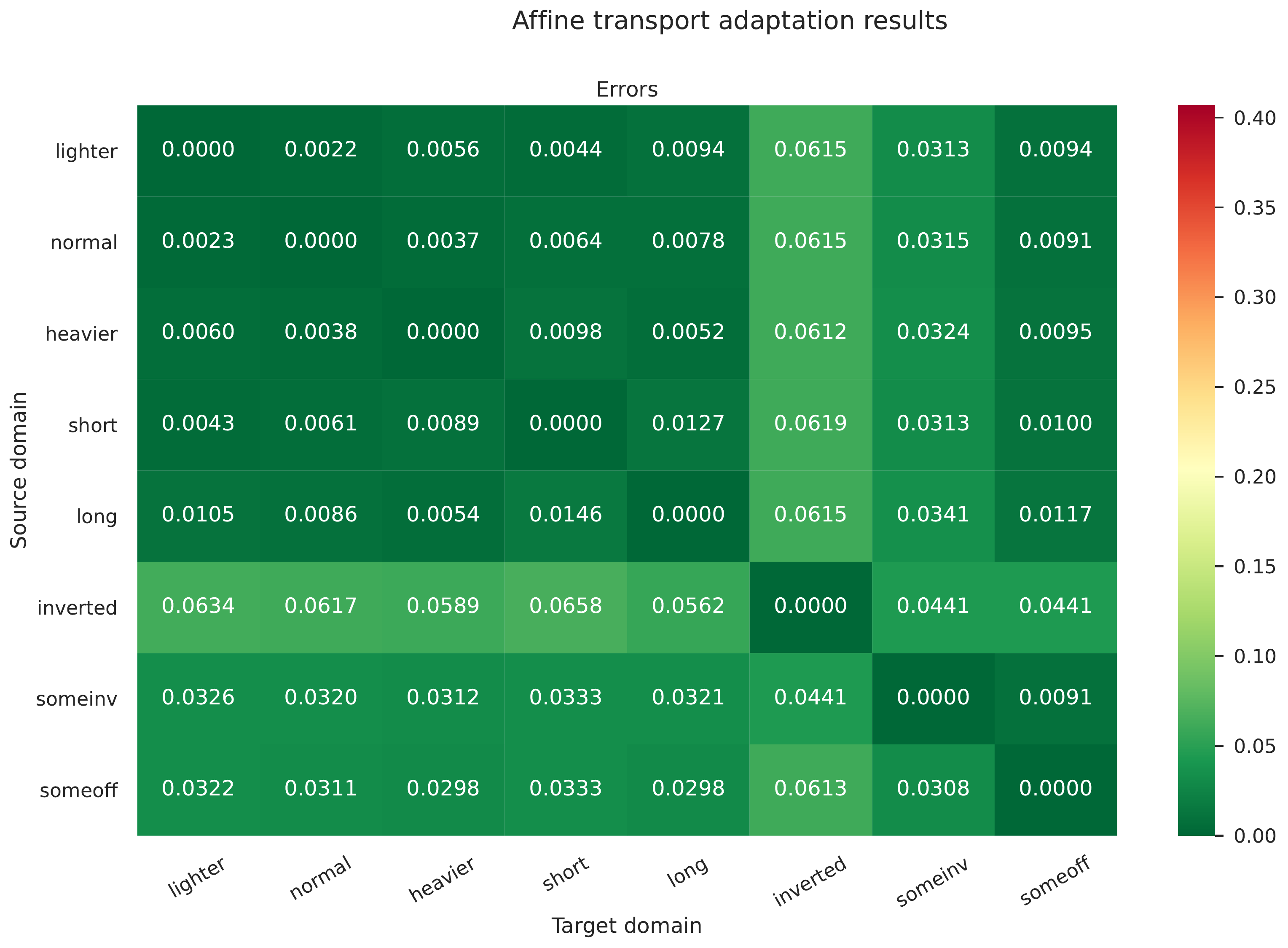}
        \caption{}
        \label{fig:hopper_errors}
    \end{subfigure}
    \caption{Results in the Hopper environment: before adaptation~(\subref{fig:hopper_prior}), the $\rho_{\aff}$ values ~(\subref{fig:hopper_rhos}), and the errors after adaptation~(\subref{fig:hopper_errors})}
    \label{fig:hopper_res}
\end{figure}

\subsection{Hockey Puck}
We experiment with the data describing a robot hand hitting a puck with a stick, trying to reach a specific point with the puck. We have simulated data $X_\simu$ and real data $X_\real$ consisting of $4$-tuples $(x,y,z_0,z_1)$, where $(x,y)$ is the end location of the puck and $(z_0,z_1)$ are latent variables encoding the action, which are the same for the simulated and real cases. In the real environment, we use two different pucks, a blue one and a red one. In the simulated case, we use pucks with $11$ different isotropies. In Table~\ref{tab:puck_sim2real} we present the prior $2$-Wasserstein distance between the simulated and real data, the $2$-Wasserstein distance between the affinely transferred data and real data, and the $\rho_\aff(X_\simu,X_\real)$ score. The transfer is illustrated in Figure~\ref{fig:transfer_example}.

In Figure~\ref{fig:learning_curve}, we vary the amount of real and simulated data points, sharing the same actions, used to learn the Sim2Real transport map $T_\aff$. 

\begin{table}
    \centering
    \begin{tabular}{lccc}
    \toprule
     Blue Puck &Sim2Real Error & Sim Error & $\rho_\aff$  \\
    \midrule
    isotr\_low          &    $0.08 \pm 0.07$ &    $0.72 \pm 0.37$ &    0.93\\
    isotr\_lowmed       &    $0.08 \pm 0.07$ &    $0.19 \pm 0.15$ &    0.93\\
    isotr\_med          &    $0.08 \pm 0.07$ &    $0.17 \pm 0.12$ &    0.93\\
    isotr\_medhi        &    $0.09 \pm 0.08$ &    $0.28 \pm 0.19$ &    0.92\\
    isotr\_high         &    $0.10 \pm 0.08$ &    $0.33 \pm 0.21$ &    0.92\\
    isotr\_low\_offc    &    $0.08 \pm 0.07$ &    $0.73 \pm 0.43$ &    0.93\\
    isotr\_high\_offc   &    $0.15 \pm 0.10$ &    $0.46 \pm 0.26$ &    0.88\\
    heavy\_anisotr\_lowx&    $0.09 \pm 0.07$ &    $0.36 \pm 0.20$ &    0.92\\
    heavy\_anisotr\_lowy&    $0.09 \pm 0.08$ &    $0.33 \pm 0.20$ &    0.92\\
    low\_anisotr\_lowx  &    $0.08 \pm 0.07$ &    $0.16 \pm 0.11$ &    0.93\\
    heavy\_anisotr\_lowy&    $0.08 \pm 0.08$ &    $0.18 \pm 0.16$ &    0.93\\
    \toprule
    Red Puck &Sim2Real Error & Sim Error & $\rho_\aff$  \\
    \midrule
    isotr\_low          &    $0.07 \pm 0.06$ &    $0.69 \pm 0.39$ &    0.92\\
    isotr\_lowmed       &    $0.07 \pm 0.06$ &    $0.18 \pm 0.17$ &    0.93\\
    isotr\_med          &    $0.07 \pm 0.06$ &    $0.18 \pm 0.12$ &    0.93\\
    isotr\_medhi        &    $0.08 \pm 0.06$ &    $0.29 \pm 0.16$ &    0.93\\
    isotr\_high         &    $0.08 \pm 0.07$ &    $0.35 \pm 0.19$ &    0.92\\
    isotr\_low\_offc    &    $0.07 \pm 0.06$ &    $0.69 \pm 0.43$ &    0.92\\
    isotr\_high\_offc   &    $0.10 \pm 0.07$ &    $0.48 \pm 0.22$ &    0.91\\
    heavy\_anisotr\_lowx&    $0.08 \pm 0.06$ &    $0.36 \pm 0.21$ &    0.92\\
    heavy\_anisotr\_lowy&    $0.07 \pm 0.06$ &    $0.33 \pm 0.19$ &    0.93\\
    low\_anisotr\_lowx  &    $0.07 \pm 0.06$ &    $0.18 \pm 0.11$ &    0.93\\
    heavy\_anisotr\_lowy&    $0.07 \pm 0.06$ &    $0.19 \pm 0.15$ &    0.93\\
    \bottomrule
    \end{tabular}
    \caption{Pointwise error for the predicted state for Sim2Real and Simulation, as well as the estimated affinity score $\rho_\aff(X,Y)$ between the simulated $X$ and real $Y$ datasets.}
    \label{tab:puck_sim2real}
\end{table}

\begin{figure*}
    \centering
    \includegraphics[width=\linewidth]{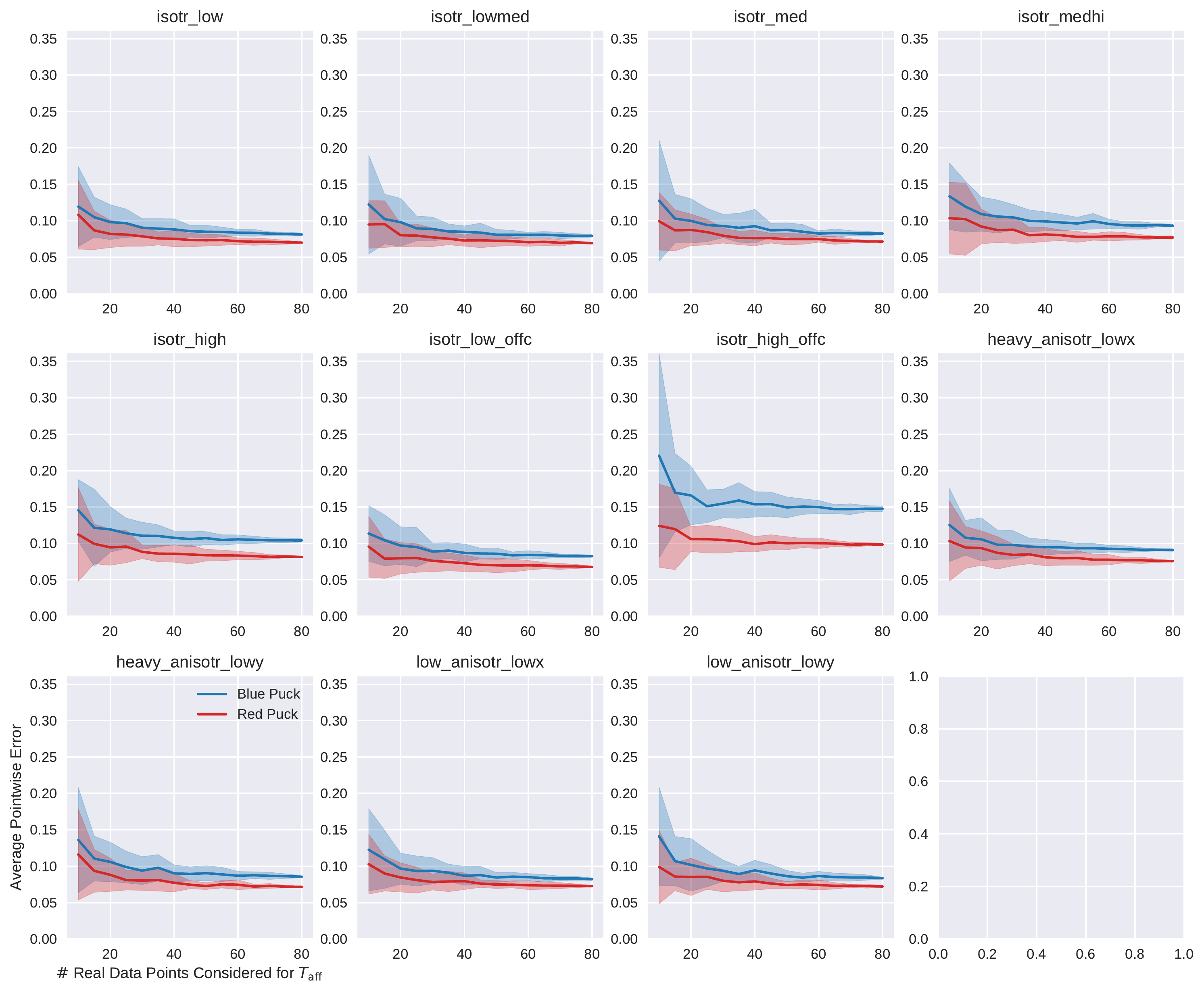}
    \caption{Pointwise transfer error of AT for different simulations as we vary the amount of real data points used to compute the AT map. The shaded area represent the  $2\sigma$ deviations.}
    \label{fig:learning_curve}
\end{figure*}

\section{Conclusion}
In this paper, we introduced affine transport and showed that it can be used to improve dynamics model performance through domain adaptation in robotics tasks, both in simulation and from simulation to reality.
The method works well in low-dimensional state and action spaces, such as in case of hockey puck; it, however, struggles to generalize to new data in higher dimensional spaces, such as in the case of Hopper.
Overall, the method can serve as a baseline for further research in domain adaptation through optimal transport.

While the method successfully adapts the dynamic model, the local character of the fit makes it potentially problematic for use with reinforcement learning. 
In that scenario, a transport map fit in on an offline dataset, a potential distribution shift could pull the reinforcement learning algorithm towards solutions which may look good under under the model, but only so because these transition were not observed in the adaptation dataset.
The simplest solution would be to update the transport map online; this, however, would attempt to fit a global affine model, which may be too simple to capture the true relations between source and target domains in all but the simplest cases.
More intricate solution would attempt to prevent distribution shift from occuring in the first place; in that case, methods studied in  the context of offline reinforcement learning could prove to be useful~\cite{yu2020mopo,  levine2020offline}.

\section*{Acknowledgment}
This work was supported by the Academy of Finland (Flagship programme: Finnish Center for Artificial Intelligence FCAI and grants 294238, 319264, 292334, 334600, 328400). We acknowledge the computational resources provided by Aalto Science-IT project.

{\small
\bibliographystyle{ieeetr}
\bibliography{references}

\begin{thebibliography}{10}

\bibitem{kober2013reinforcement}
J.~Kober, J.~A. Bagnell, and J.~Peters, ``Reinforcement learning in robotics: A
  survey,'' {\em The International Journal of Robotics Research}, vol.~32,
  no.~11, pp.~1238--1274, 2013.

\bibitem{wilson2020survey}
G.~Wilson and D.~J. Cook, ``A survey of unsupervised deep domain adaptation,''
  {\em ACM Transactions on Intelligent Systems and Technology (TIST)}, vol.~11,
  no.~5, pp.~1--46, 2020.

\bibitem{redko2020survey}
I.~Redko, E.~Morvant, A.~Habrard, M.~Sebban, and Y.~Bennani, ``A survey on
  domain adaptation theory: learning bounds and theoretical guarantees,'' {\em
  arXiv e-prints}, pp.~arXiv--2004, 2020.

\bibitem{zhao2020sim}
W.~Zhao, J.~P. Queralta, and T.~Westerlund, ``Sim-to-real transfer in deep
  reinforcement learning for robotics: a survey,'' in {\em 2020 IEEE Symposium
  Series on Computational Intelligence (SSCI)}, pp.~737--744, IEEE, 2020.

\bibitem{courty17DA}
N.~{Courty}, R.~{Flamary}, D.~{Tuia}, and A.~{Rakotomamonjy}, ``Optimal
  transport for domain adaptation,'' {\em IEEE Transactions on Pattern Analysis
  and Machine Intelligence}, vol.~39, no.~9, pp.~1853--1865, 2017.

\bibitem{courty2017joint}
N.~Courty, R.~Flamary, A.~Habrard, and A.~Rakotomamonjy, ``Joint distribution
  optimal transportation for domain adaptation,'' {\em Advances in Neural
  Information Processing Systems}, vol.~30, pp.~3730--3739, 2017.

\bibitem{redko2017theoretical}
I.~Redko, A.~Habrard, and M.~Sebban, ``Theoretical analysis of domain
  adaptation with optimal transport,'' in {\em Joint European Conference on
  Machine Learning and Knowledge Discovery in Databases}, pp.~737--753,
  Springer, 2017.

\bibitem{seguy2017large}
V.~Seguy, B.~B. Damodaran, R.~Flamary, N.~Courty, A.~Rolet, and M.~Blondel,
  ``Large-scale optimal transport and mapping estimation,'' {\em arXiv preprint
  arXiv:1711.02283}, 2017.

\bibitem{sun2015return}
B.~Sun, J.~Feng, and K.~Saenko, ``Return of frustratingly easy domain
  adaptation,'' {\em arXiv preprint arXiv:1511.05547}, 2015.

\bibitem{zhang2018aligning}
Z.~Zhang, M.~Wang, Y.~Huang, and A.~Nehorai, ``Aligning infinite-dimensional
  covariance matrices in reproducing kernel hilbert spaces for domain
  adaptation,'' in {\em Proceedings of the IEEE Conference on Computer Vision
  and Pattern Recognition}, pp.~3437--3445, 2018.

\bibitem{smith1987note}
C.~S. Smith and M.~Knott, ``Note on the optimal transportation of
  distributions,'' {\em Journal of Optimization Theory and Applications},
  vol.~52, no.~2, pp.~323--329, 1987.

\bibitem{givens84}
C.~R. Givens, R.~M. Shortt, {\em et~al.}, ``A class of {W}asserstein metrics
  for probability distributions.,'' {\em The Michigan Mathematical Journal},
  vol.~31, no.~2, pp.~231--240, 1984.

\bibitem{dowson82}
D.~Dowson and B.~Landau, ``The {F}r{\'e}chet distance between multivariate
  normal distributions,'' {\em Journal of multivariate analysis}, vol.~12,
  no.~3, pp.~450--455, 1982.

\bibitem{olkin82}
I.~Olkin and F.~Pukelsheim, ``The distance between two random vectors with
  given dispersion matrices,'' {\em Linear Algebra and its Applications},
  vol.~48, pp.~257--263, 1982.

\bibitem{knott84}
M.~Knott and C.~S. Smith, ``On the optimal mapping of distributions,'' {\em
  Journal of Optimization Theory and Applications}, vol.~43, no.~1, pp.~39--49,
  1984.

\bibitem{takatsu11}
A.~Takatsu {\em et~al.}, ``Wasserstein geometry of gaussian measures,'' {\em
  Osaka Journal of Mathematics}, vol.~48, no.~4, pp.~1005--1026, 2011.

\bibitem{gelbrich1990formula}
M.~Gelbrich, ``On a formula for the l2 wasserstein metric between measures on
  euclidean and hilbert spaces,'' {\em Mathematische Nachrichten}, vol.~147,
  no.~1, pp.~185--203, 1990.

\bibitem{brenier1991polar}
Y.~Brenier, ``Polar factorization and monotone rearrangement of vector-valued
  functions,'' {\em Communications on pure and applied mathematics}, vol.~44,
  no.~4, pp.~375--417, 1991.

\bibitem{kendall1989survey}
D.~G. Kendall, ``A survey of the statistical theory of shape,'' {\em
  Statistical Science}, pp.~87--99, 1989.

\bibitem{brockman16gym}
G.~Brockman, V.~Cheung, L.~Pettersson, J.~Schneider, J.~Schulman, J.~Tang, and
  W.~Zaremba, ``Openai gym,'' {\em CoRR}, vol.~abs/1606.01540, 2016.

\bibitem{yu2020mopo}
T.~Yu, G.~Thomas, L.~Yu, S.~Ermon, J.~Y. Zou, S.~Levine, C.~Finn, and T.~Ma,
  ``Mopo: Model-based offline policy optimization,'' in {\em Advances in Neural
  Information Processing Systems} (H.~Larochelle, M.~Ranzato, R.~Hadsell, M.~F.
  Balcan, and H.~Lin, eds.), vol.~33, pp.~14129--14142, Curran Associates,
  Inc., 2020.

\bibitem{levine2020offline}
S.~Levine, A.~Kumar, G.~Tucker, and J.~Fu, ``Offline reinforcement learning:
  Tutorial, review, and perspectives on open problems,'' {\em CoRR},
  vol.~abs/2005.01643, 2020.

\end{thebibliography}
}

\end{document}